\documentclass[twoside]{article}

\usepackage{aistats2019}

\usepackage{amsmath}
\usepackage{amsthm}
\usepackage{amssymb}
\usepackage{cases}
\usepackage{proof}

\usepackage{color}

\definecolor{mydarkblue}{rgb}{0,0,0}

\usepackage{natbib}
\setcitestyle{authoryear,round,citesep={;},aysep={,},yysep={;}}

\usepackage{hyperref}

\hypersetup{ %
    colorlinks=true,
    linkcolor=mydarkblue,
    citecolor=mydarkblue,
    filecolor=mydarkblue,
    urlcolor=mydarkblue}

\usepackage{algorithm}
\usepackage{algpseudocode}
\makeatletter
\def\BState{\State\hskip-\ALG@thistlm}
\makeatother

\usepackage{enumerate}

\usepackage{wrapfig}

\usepackage{bm, bbm}

\usepackage{subfigure}

\usepackage{booktabs}

\usepackage{multicol}

\usepackage{pgf,tikz}
\usetikzlibrary{shapes,automata,snakes,backgrounds,arrows}
\usetikzlibrary{mindmap}

\def\squareforqed{\hbox{\rlap{$\sqcap$}$\sqcup$}}
\def\qed{\ifmmode\squareforqed\else{\unskip\nobreak\hfil
\penalty50\hskip1em\null\nobreak\hfil\squareforqed
\parfillskip=0pt\finalhyphendemerits=0\endgraf}\fi}

\newtheorem{theorem}{Theorem}
\newtheorem{definition}{Definition}

\newtheorem{lemma}{Lemma}
\newtheorem{corollary}{Corollary}

\newtheorem{assumption}{Assumption}

\newcommand{\real}[1]{\mathbb{R}^{{#1}}}
\newcommand{\innerprod}[1]{\langle {#1} \rangle}
\newcommand{\norm}[1]{\lVert {#1} \rVert}

\newcommand{\Norm}[1]{\left\lVert {#1} \right\rVert}

\newcommand{\prox}[2][{}]{\mathbf{prox}^{#1}_{#2}}

\renewcommand{\cite}{\citep}

\begin{document}

\twocolumn[

\aistatstitle{A Model Parallel Proximal Stochastic Gradient Algorithm for Partially Asynchronous Systems}

\aistatsauthor{ Rui~Zhu \And Di~Niu}

\aistatsaddress{Department of Electrical and Computer Engineering, University of Alberta}

]


\begin{abstract}
Large models are prevalent in modern machine learning scenarios, including deep learning, recommender systems, etc., which can have millions or even billions of parameters. Parallel algorithms have become an essential solution technique to many large-scale machine learning jobs. In this paper, we propose a model parallel proximal stochastic gradient algorithm, AsyB-ProxSGD, to deal with large models using model parallel blockwise updates while in the meantime handling a large amount of training data using proximal stochastic gradient descent (ProxSGD). In our algorithm, worker nodes communicate with the parameter servers asynchronously, and each worker performs proximal stochastic gradient for only one block of model parameters during each iteration. Our proposed algorithm generalizes ProxSGD to the asynchronous and model parallel setting. We prove that AsyB-ProxSGD achieves a convergence rate of $O(1/\sqrt{K})$ to stationary points for nonconvex problems under \emph{constant} minibatch sizes, where $K$ is the total number of block updates. This rate matches the best-known rates of convergence for a wide range of gradient-like algorithms. Furthermore, we show that when the number of workers is bounded by $O(K^{1/4})$, we can expect AsyB-ProxSGD to achieve linear speedup as the number of workers increases. We implement the proposed algorithm on MXNet and demonstrate its convergence behavior and near-linear speedup on a real-world dataset involving both a large model size and large amounts of data.
\end{abstract}


\section{Introduction}
\label{sec:intro}

Many machine learning problems can be formulated as the following general minimization framework:
\begin{equation}
\begin{split}
	\mathop{\min}_{x\in \real{d}} &\quad \Psi(x):= \frac{1}{n}\sum_{i=1}^n f_i(x) + h(x),
\end{split}
\label{eq:original}
\end{equation}
where $f_i(x)$ is typically a smooth yet possibly nonconvex loss function of the $i$-th sample, and $h(x)$ is a convex yet nonsmooth regularizer term that promotes some structures. Examples include deep learning with regularization \cite{dean2012large,chen2015mxnet,zhang2015deep}, LASSO \cite{tibshirani2005sparsity}, sparse logistic regression \cite{liu2009large}, robust matrix completion \cite{xu2010robust,sun2015guaranteed}, and sparse support vector machine (SVM) \cite{friedman2001elements}.

Many classical deterministic (non-stochastic) algorithms are available to solve problem \eqref{eq:original}, including the proximal gradient (ProxGD) method \cite{parikh2014proximal} and its accelerated variants \cite{li2015accelerated} as well as the alternating direction method of multipliers (ADMM) \cite{hong2016convergence}. These methods leverage the so-called \emph{proximal operators} \cite{parikh2014proximal} to handle the nonsmoothness in the problem. However, these algorithms require calculating the gradient of all $n$ samples in each iteration, which is expensive in modern machine learning problems. The trend to deal with large volumes of data is the use of \emph{stochastic} algorithms. As the number of training samples $n$ increases, the cost of updating the model $x$ taking into account all error gradients becomes prohibitive. To tackle this issue, stochastic algorithms make it possible to update $x$ using only a small subset of all training samples at a time.

Stochastic gradient descent (SGD) is one of the first algorithms widely implemented in an asynchronous parallel fashion; its convergence rates and speedup properties have been analyzed for both convex \cite{agarwal2011distributed,mania2017} and nonconvex \cite{lian2015asynchronous} optimization problems.
Nevertheless, SGD is mainly applicable to the case of smooth optimization, and yet is not suitable for problems with a \emph{nonsmooth} term in the objective function, e.g., an $\ell_1$ norm regularizer. In fact, such nonsmooth regularizers are commonplace in many practical machine learning problems or constrained optimization problems. In these cases, SGD becomes ineffective, as it is hard to obtain gradients for a nonsmooth objective function.

With rapidly growing data volumes and model complexity, the need to scale up machine learning has sparked broad interests in developing efficient parallel optimization algorithms.
A typical parallel optimization algorithm usually decomposes the original problem into multiple subproblems, each handled by a worker node.
Each worker iteratively downloads the global model parameters and computes its local gradients to be sent to the master node or servers for model updates. Recently, asynchronous parallel optimization algorithms \cite{recht2011hogwild,li2014communication,lian2015asynchronous}, exemplified by the Parameter Server architecture \cite{li2014scaling}, have been widely deployed in industry to solve practical large-scale machine learning problems. Asynchronous algorithms can largely reduce overhead and speedup training, since each worker may individually perform model updates in the system without synchronization.

Existing parallel algorithms fall into two categories: \emph{data parallelism} and \emph{model parallelism}. In data parallelism, each worker takes a subset of training samples $i$ and calculates their loss functions $f_i$'s and/or gradients in parallel. For example, a typical implementation of parallel SGD is to divide a minibatch with $N$ samples into several smaller minibatches (each with $N'$ samples), and each worker computes gradients on $N'$ samples. This is preferred when the size of data $n$ is large. In model parallelism, the model parameters $x$ is partitioned into $M$ blocks, where $x_j\in \real{d_j}$ with $d_j \in \mathbb{N}_+$ and $\sum_{j=1}^M d_j=d$. Since proximal operator on $x$ can be decomposed into those on individual blocks \cite{parikh2014proximal}, proximal gradient and its stochastic version (proximal stochastic gradient descent or ProxSGD) is again a natural candidate to solve \eqref{eq:original}. \cite{zhou2016convergence} shows that Block Proximal Gradient Descent works in an asynchronous parallel mode, and in the meantime proves that Block ProxSGD can converge to critical points when the maximum staleness is bounded. However, theoretical understanding of the behavior of Asynchronous Block ProxSGD, a more useful algorithm in practical Parameter Servers, is a gap yet to be filled.

In this paper, we propose AsyB-ProxSGD (Asynchronous Block Proximal Stochastic Gradient Descent), an extension of proximal stochastic gradient (ProxSGD) algorithm to the \emph{model parallel} paradigm and to the \emph{partially asynchronous protocol} (PAP) setting. In AsyB-ProxSGD, workers asynchronously communicate with  the parameter servers, which collectively store model parameters in blocks. In an iteration, each worker pulls the latest yet possibly outdated model from servers, calculates partial gradients for only one block based on stochastic samples, and pushes the gradients to the corresponding server. As workers can update different blocks in parallel, AsyB-ProxSGD is different from traditional data parallel ProxSGD can handle both a large model size $d$ and a large number $n$ of training samples, a case frequently observed in reality.

Our theoretical contribution is summarized as follows. 
We prove that AsyB-ProxSGD can converge to stationary points of the nonconvex and nonsmooth problem \eqref{eq:original} with an ergodic convergence rate of $O(1/\sqrt{K})$, where $K$ is the total number of times that any block in $x$ is updated. This rate matches the convergence rate known for asynchronous SGD. The latter, however, is suitable only for smooth problems. To our best knowledge, this is the first work that provides convergence rate guarantees for ProxSGD in a model parallel mode, especially in an asynchronous setting. We also provide a linear speedup guarantee as the number of workers increases, provided that the number of workers is bounded by $O({K}^{1/4})$. This result has laid down a theoretical ground for the scalability and performance of AsyB-ProxSGD in practice. Evaluation based on a real-world dataset involving both a large model and a large dataset has corroborated our theoretical findings on the convergence and speedup behavior of AsyB-ProxSGD, under a Parameter Server implementation.  

\section{Preliminaries}
\label{sec:prelim}

In this section, we first introduce some notations to be used throughout the paper. Then we introduce the stochastic optimization problem to be studied. Finally, we introduce proximal operators and enumerate fundamental assumptions made in the model.

We use $\norm{x}$ to denote the $\ell_2$ norm of the vector $x$, and $\innerprod{x, y}$ to denote the inner product of two vectors $x$ and $y$. We use $g(x)$ to denote the ``true'' gradient $\nabla f(x)$ and use $G(x;\xi)$ to denote the stochastic gradient $\nabla F(x;\xi)$ for a function $f(x)$. 
Let $\nabla_j$ be the derivative w.r.t. $x_j$, the $j$-th coordinate of $x$; and let $G_j$ and $g_j$ represent $\nabla_j F(x; \xi)$ and $\nabla_j f(x)$, respectively. For a random variable or vector $X$, let $\mathbb{E}[X|\mathcal{F}]$ be the conditional expectation of $X$ w.r.t. a sigma algebra $\mathcal{F}$.

\subsection{Stochastic Optimization Problems}
In this paper, we consider the following \emph{stochastic} optimization problem instead of the original deterministic version \eqref{eq:original}:
\begin{equation}
\begin{split}
	\mathop{\min}_{x\in \real{d}} &\quad \Psi(x):= \mathbb{E}_\xi [F(x; \xi)] + h(x),
\end{split}
\label{eq:stochastic}
\end{equation}
where the stochastic nature comes from the random variable $\xi$, which in our problem settings, represents a random index selected from the training set $\{1, \ldots, n\}$. Therefore, \eqref{eq:stochastic} attempts to minimize the expected loss of a training sample plus a regularizer $h(x)$. When it comes to large models, we decompose $x$ into $M$ blocks, and rewrite \eqref{eq:stochastic} into the following block optimization form:
\begin{equation}
\mathop{\min}_{x \in \real{d}} \Psi(x) := \mathbb{E}_\xi [F(x_1,\ldots,x_M; \xi)] + \sum_{j=1}^M h_j(x_j),
\label{eq:stochastic_block}
\end{equation}
where $x=(x_1, \ldots, x_M)$, $x_j \in \real{d_j}$ for those $d_j \in \mathbb{N}_+$ and $\sum_{j=1}^M d_j = d$, and $h(x)= \sum_{j=1}^M h_j(x_j)$.

In this work, we assume all $h_j$'s for problem~\eqref{eq:stochastic_block} are proper, closed and convex, yet \emph{not necessarily smooth}. To handle the potential non-smoothness, we introduce the following generalized notion of derivatives to be used in convergence analysis.

\begin{definition}[Subdifferential e.g.,~\cite{parikh2014proximal}]
	We say a vector $p \in \real{d}$ is a subgradient of the function $h: \real{d} \to \real{}$ at $x \in \mathrm{dom}\ h$, if for all $z \in \mathrm{dom}\ h$,
	\begin{equation}
		h(z) \geq h(x) + \innerprod{p, z-x}.
	\end{equation}
Moreover, denote the set of all such subgradients at $x$ by $\partial h(x)$, which is called the subdifferential of $h$ at $x$.
\end{definition}
For problems~\eqref{eq:stochastic} and ~\eqref{eq:stochastic_block}, we define the critical point as follows:
\begin{definition}[Critical point \cite{attouch2013convergence}]
	A point $x\in \real{d}$ is a critical point of $\Psi$, iff $0 \in \nabla f(x) + \partial h(x)$.
\end{definition}

\subsection{Proximal Gradient Descent}

The proximal operator is fundamental to many algorithms to solve problem \eqref{eq:original} as well as its stochastic variants \eqref{eq:stochastic} and \eqref{eq:stochastic_block}.

\begin{definition}[Proximal operator]
	The proximal operator $\prox{}$ of a point $x \in \real{d}$ under a proper and closed function $h$ with parameter $\eta > 0$ is defined as:
\begin{equation}
	\prox{\eta h}(x) = \mathop{\arg \min}_{y \in \real{d}} \left\{h(y) + \frac{1}{2\eta} \norm{y-x}^2\right\}.
\end{equation}
\end{definition}
In its vanilla version, \emph{proximal gradient descent} performs the following iterative updates:
\begin{equation*}
	x^{k+1} \gets \prox{\eta_k h}(x^{k} - \eta_k \nabla f(x^k)),
\end{equation*}
for $k=1,2,\ldots$, where $\eta_k > 0$ is the step size at iteration $k$. 

To solve stochastic optimization problems \eqref{eq:stochastic} and \eqref{eq:stochastic_block}, we need a variant called \emph{proximal stochastic gradient descent} (ProxSGD), with its update rule at iteration $k$ given by
\begin{equation*}
	x^{k+1} \gets \prox{\eta_k h}\left( x^{k} - \frac{\eta_k}{\lvert \Xi_k \rvert}\sum_{\xi \in \Xi_k}\nabla F(x^k; \xi) \right).
\end{equation*}
In ProxSGD, the gradient $\nabla f$ is replaced by the gradients from a random subset of training samples, denoted by $\Xi_k$ at iteration $k$. Since $\xi$ is a random variable indicating a random index in $\{1,\ldots, n\}$, $F(x; \xi)$ is a random loss function for the random sample $\xi$, such that $f(x) := \mathbb{E}_\xi [F(x; \xi)]$.

With these definitions, we now introduce our metric used in ergodic convergence analysis:
\begin{equation*}
	P(x,g,\eta) := \frac{1}{\eta}(x - \prox{\eta h}(x - \eta g)),
\end{equation*}
which is also called the \emph{gradient mapping} in the literature, e.g., \cite{parikh2014proximal}. For non-convex problems, it is a standard approach to measure convergence (to a stationary point) by gradient mapping according to the following lemma:
\begin{lemma}[Non-convex Convergence \cite{attouch2013convergence}]
	A point $x$ is a critical point of \eqref{eq:stochastic} iff. $P(x,g,\eta) = 0$.
\end{lemma}
Therefore, we can use the following definition as a convergence metric:
\begin{definition}[Iteration complexity \cite{reddi2016proximal}]
	A solution $x$ is called $\epsilon$-accurate, if $\mathbb{E}[\norm{P(x,g,\eta)}^2] \leq \epsilon$ for some $\eta > 0$. If an algorithm needs at least $K$ iterations to find an $\epsilon$-accurate solution, its iteration complexity is $K$.
\end{definition}

We make the following assumptions throughout the paper. Other algorithm specific assumptions will be introduced later in the corresponding sections.

We assume that $f(\cdot)$ is a smooth function with the following properties:
\begin{assumption}[Lipschitz Gradient]
	For function $f$ there are Lipschitz constants $L_j, L>0$ such that
	\begin{align}
		\norm{\nabla f(x) - \nabla f(y)} \leq L \norm{x - y}, \forall x, y \in \real{d}, \\
		\norm{\nabla_j f(x) - \nabla_j f(x+\alpha e_j)} \leq L_{\max} |\alpha|, \forall x \in \real{d},
	\end{align}
	where $e_j$ is an indicator vector that is only valid at block $j$ with value 1, and vanishes to zero at other blocks. Clearly we have $ L_{\max} \leq L$.
	\label{asmp:smooth}
\end{assumption}
As discussed above, assume that $h$ (or $h_j$) is a proper, closed and convex function, which is yet not necessarily smooth.
If the algorithm has been executed for $k$ iterations, we let $\mathcal{F}_k$ denote the set that consists of all the samples used up to iteration $k$. Since $\mathcal{F}_k \subseteq \mathcal{F}_{k'}$ for all $k \leq k'$, the collection of all such $\mathcal{F}_k$ forms a \emph{filtration}. Under such settings, we can restrict our attention to those stochastic gradients with an unbiased estimate and bounded variance, which are common in the analysis of \emph{stochastic} gradient descent or \emph{stochastic} proximal gradient algorithms, e.g., \cite{lian2015asynchronous, ghadimi2016mini}.
\begin{assumption}[Unbiased gradient]
	For any $k$, we have $\mathbb{E}[G_k | \mathcal{F}_{k}] = g_k$.
	\label{asmp:unbias_grad}
\end{assumption}

\begin{assumption}[Bounded variance]
	The variance of the stochastic gradient is bounded by $\mathbb{E}[\norm{G(x;\xi) - \nabla f(x)}^2] \leq \sigma^2$.
	\label{asmp:var_grad}
\end{assumption}

\subsection{Parallel Stochastic Optimization}

Recent years have witnessed rapid development of parallel and distributed computation frameworks for large-scale machine learning problems. One popular architecture is called \emph{parameter server} \cite{dean2012large,li2014scaling}, which consists of some worker nodes and server nodes. In this architecture, one or multiple master machines play the role of parameter servers, which maintain the model $x$. All other machines are \emph{worker nodes} that communicate with servers for training machine learning models. In particular, each worker has two types of requests: \texttt{pull} the current model $x$ from servers, and \texttt{push} the computed gradients to servers.

Before proposing our AsyB-ProxSGD algorithm in the next section, let us first introduce its \emph{synchronous} version. Suppose we execute ProxSGD with a minibatch of 128 random samples on 8 workers and our goal is to use these 8 workers to train a model with 8 blocks\footnote{For brevity, we assume that each server hosts only one block, and we will say server $j$ and block $j$ interchangeably in this paper.}. We can let each worker randomly take 128 samples, compute a summed partial gradient w.r.t one block (say, block $j$) on them, and push it to server $j$. In synchronous case, server $j$ will finally receive 8 summed partial gradients on 8 blocks (each partial gradient contains information of 128 samples) in each iteration. This server then updates the model by performing the proximal gradient descent step. In general, if we divide the model into $M$ blocks, each worker will be assigned to update only one block using a minibatch of $N$ samples and they can do updating in parallel in an iteration.

Note that in this scenario, all workers have to calculate the whole minibatch of the computation. Thanks to the decomposition property of proximal operator \cite{parikh2014proximal}, updating blocks can be done in parallel by servers, which corresponds to \emph{model parallelism} in the literature (e.g., \cite{recht2011hogwild,zhou2016convergence,pan2016cyclades}). Another type of parallelism is called \emph{data parallelism}, in which each worker uses only part of $N$ random samples in the minibatch to compute a full gradient on $x$ (e.g., \cite{agarwal2011distributed,ho2013more}). We handle the issue of large $n$ by using stochastic algorithms, and our main focus is to handle the large model challenge by model parallelism.

\section{AsyB-ProxSGD: Asynchronous Block Proximal Stochastic Gradient}
\label{sec:aspg}

We now present our main contribution in this paper, \emph{Asynchronous Block Proximal Stochastic Gradient} (AsyB-ProxSGD) algorithm. Recall that asynchronous algorithm tries to alleviate random delays in computation and communication in different iterations. When model is big, it is hard to put the whole model in a single node (a single machine or device), and we have to split it into $M$ blocks. In this case, no single node maintains all of the parameters in memory and the nodes can update in parallel. The idea of model parallelism has been used in many applications, including deep learning \cite{dean2012large} and factorization machine \cite{li2016difacto}.

We now formally introduce how our proposed algorithm works. The main idea of our proposed algorithm is to update block $x_j$ in parallel by different workers. In Algorithm~\ref{alg:asyn-prox-scd}, the first step is to ensure that the staleness is upper bounded by $T$, which is essential to ensure convergence. 
Here we use $\hat{x}$ to emphasize that the pulled model parameters $x$ may not be consistent with that stored on parameter servers. Since blocks are scattered on multiple servers, different blocks may be not consistent with updates and thus results in different delays. For example, suppose the server stores model $x=(x_1, x_2)$, and we have two workers that updates $x_1$ and $x_2$ in parallel. Our expectation is that $x$ is updated by them and it becomes $x'=(x_1', x_2')$. However, in partially asynchronous protocol (PAP) where workers may skip synchronization, the following case may happen. At time 1, worker 1 pushes $x_1'$ and pulls $x_2$; thus, worker 1 gets $(x_1', x_2)$. At time 2, worker 2 pushes $x_2'$ and pulls $x_1'$; thus, worker 2 gets $(x_1', x_2')$. We can see that the next update by worker 1 is based on $(x_1', x_2)$, which has different delays on two blocks.

Let us discuss this in more implementation details for distributed clusters. In distributed clusters, we split a large model $x$ into $M$ blocks, and one server only maintains a single block $x_j$ to achieve model parallelism. Thus, different block may be updated at different iterations by different workers. The same phenomenon also exist in shared memory systems (i.e., a single machine with multiple CPU cores or GPUs, etc.). In these systems, the model is stored on the main memory and we can regard it as a ``logical'' server. In these systems, ``reading'' and ``writing'' can be done simultaneously, thus block $x_j$ may be ``pulled'' while it is being updated. In summary, model parameters $x$ may be inconsistent with any actual state on the server side.

In our algorithm, workers can update multiple blocks in parallel, and this is the spirit of model parallelism here. However, we note that on the server side, \texttt{push} request is usually more time consuming than \texttt{pull} request since it needs additional computations of the proximal operator. 
Therefore, we should let workers gather more stochastic gradients before pushing to the sever, and that is the reason we let each worker to compute gradients on all $N$ samples in a minibatch. That is, a worker iteration $t$ should compute
\[\hat{G}_{j_t}^t := \frac{1}{N}\sum_{i=1}^N \hat{G}_{j_t}(\hat{x}^t;\xi_{i,t}),\]
where $j_t$ is the index of block to be updated at iteration $t$, and $\hat{G}_{j_t}(\hat{x}^t;\xi_{i,t})$ is the partial gradient w.r.t. block $j_t$ at model $\hat{x}^t$ pulled at iteration $t$ and on sample $\xi_{i,t}$.


\begin{algorithm}[thpb]
\caption{AsyB-ProxSGD: Block PAP Stochastic Gradient}
\label{alg:asyn-prox-scd}
\underline{\textbf{Server $j$ executes:}}
\begin{algorithmic}[1]
	\State Initialize $x^0$.
	\Loop
		\If {\texttt{Pull Request} from a worker is received}
			\State Send $x_j$ to the worker.
		\EndIf
		\If {\texttt{Push Request} (gradient $G_j$) from a worker is received}
			\State $x_j \gets \prox{\eta h_j}(x_j - \eta G_j)$.
		\EndIf
	\EndLoop
\end{algorithmic}
\underline{\textbf{Worker asynchronously performs on block $j$:}}
\begin{algorithmic}[1]
\State Pull $x^0$ to initialize.
\For {$t=0,1,\ldots$}
	\State Wait until all iterations before $t-T$ are finished at all workers.
	\State Randomly choose $N$ training samples indexed by $\xi_{t,1},\ldots,\xi_{t,N}$.
	\State Calculate $G^t_j = \frac{1}{N}\sum_{i=1}^{N} \nabla_j F(x^t; \xi_{t, i})$.
	\State Push $G^t_j$ to server $j$.
	\State Pull the current model $x$ from servers: $x^{t+1} \gets x$.
\EndFor
\end{algorithmic}
\end{algorithm}


\section{Convergence Analysis}
\label{sec:theory}

To facilitate the analysis of Algorithm~\ref{alg:asyn-prox-scd}, we rewrite it in an equivalent global view (from the server's perspective), as described in Algorithm~\ref{alg:apscd-global}. In this algorithm, we define one \emph{iteration} as the time to update any \emph{single} block of $x$ and to successfully store it at the corresponding server. We use a counter $k$ to record how many times the model $x$ has been updated; $k$ increments every time a push request (model update request) is completed for a \emph{single} block. Note that such a counter $k$ is \emph{not} required by workers to compute gradients and is different from the counter $t$ in Algorithm~\ref{alg:asyn-prox-scd}---$t$ is maintained by each worker to count how many sample gradients have been computed locally.

In particular, for every worker, it takes $N$ stochastic sample gradients and aggregates them by averaging:
{
\begin{equation}
	\hat{G}_{j_k}^k := \frac{1}{N}\sum_{i=1}^N \nabla_{j_k} F(
		x^{k-\mathbf{d}_k}; \xi_{k, i}),
	\label{eq:grad_aspg}
\end{equation}
}
where $j_k$ is the random index chosen at iteration $k$, $\mathbf{d}_k=(d_{k,1},\ldots,d_{k, M})$ denotes the \emph{delay vector}, i.e., the delays of different blocks in $\hat{x}^k$ when computing the gradient for sample $\xi_{k, i}$ at iteration $k$, and $d_{k,j}$ is the delay of a specific block $x_j$. In addition, we denote $\hat{x}:=x^{k-\mathbf{d}_k} := (x_1^{k-d_{k,1}}, \ldots, x_M^{k-d_{k, M}})$ as a vector of model parameters pulled from the server side. Then, the server updates $x^k$ to $x^{k+1}$ using proximal gradient descent.


\begin{algorithm}[htpb]
\caption{AsyB-ProxSGD (from a Global Perspective)}
\label{alg:apscd-global}
\begin{algorithmic}[1]
\State Initialize $x^1$.
\For {$k=1,\ldots, K$}
	\State {Randomly select $N$ training samples indexed by $\xi_{k, 1}, \ldots, \xi_{k, N}$.}
	\State {Randomly select a coordinate index $j_k$ from $\{1, \ldots, M\}$.}
	\State {Calculate the averaged gradient $\hat{G}_{j_k}^k$ according to \eqref{eq:grad_aspg}.}
	\For {$j=1,\ldots, M$}
		\If {$j = j_k$}
			\State {$x_j^{k+1} \gets \prox{\eta_k h_j}(x_j^{k} - \eta_k \hat{G}_j^{k})$.}
		\Else
			\State {$x_j^{k+1} \gets x_j^k $.}
		\EndIf
	\EndFor
\EndFor
\end{algorithmic}
\end{algorithm}

\subsection{Assumptions and Metrics}
To analyze Algorithm \eqref{alg:apscd-global}, we make the following common assumptions on the delay and independence \cite{recht2011hogwild,liu2015asynchronous,avron2015revisiting}:
\begin{assumption}[Bounded delay]
	There exists an constant $T$ such that for all $k$, all values in delay vector $\mathbf{d}_k$ are upper bounded by $T$: $0\leq d_{k,j} \leq T$ for all $j$.
	\label{asmp:bound1}
\end{assumption}
\begin{assumption}[Independence]
	All random variables including selected indices $\{j_k\}$ and samples $\{\xi_{k,i}\}$ for all $k$ and $i$ in Algorithm~\ref{alg:apscd-global} are mutually independent.
	\label{asmp:indie1}
\end{assumption}

The assumption of bounded delay is to guarantee that gradients from workers should not be too old. Note that the maximum delay $T$ is roughly \emph{proportional to the number of workers} in practice. We can enforce all workers to wait for others if it runs too fast, like step 3 of workers in Algorithm~\ref{alg:asyn-prox-scd}. This setting is also called \emph{partially synchronous parallel} \cite{ho2013more,li2014communication,zhou2016convergence} in the literature.
Another assumption on independence can be met by selecting samples with \emph{replacement}, which can be implemented using some distributed file systems like HDFS \cite{borthakur2008hdfs}. These two assumptions are common in convergence analysis for asynchronous parallel algorithms, e.g., \cite{lian2015asynchronous,davis2016sound}.

\subsection{Theoretical Results}
We present our main convergence theorem as follows:

\begin{theorem}
	If the step length sequence $\{\eta_k\}$ in Algorithm~\ref{alg:asyn-prox-scd} satisfies
	{
	\begin{equation}
		\eta_k \leq \frac{1}{16L_{\max}},\quad 6\eta_k L^2 T \sum_{l=1}^T \eta_{k+l} \leq M^2,
	\end{equation}
	}
	for all $k=1,2,\ldots, K$, we have the following ergodic convergence rate for Algorithm ~\ref{alg:apscd-global}:
	{\small
	\begin{equation}
	\begin{split}
&\quad\ \frac{\sum_{k=1}^K (\eta_k - 8L_{\max}\eta_k^2) \norm{P(x^k, g^k, \eta_k)}^2}{\sum_{k=1}^K \eta_k - 8L_{\max}\eta_k^2} \\
& \leq \frac{8M(\Psi(x^1) - \Psi(x^*))}{\sum_{k=1}^K \eta_k - 8L_{\max}\eta_k^2} \\
&\quad + \frac{8M \sum_{k=1}^K \left( \frac{L\eta_k^2}{MN} + \frac{3\eta_kL^2 T\sum_{l=1}^T \eta_{k-l}^2}{2M^3 N} \right) \sigma^2}{\sum_{k=1}^K \eta_k - 8L_{\max}\eta_k^2},
	\end{split}
	\end{equation}
	}
	where the expectation is taken in terms of all the random variables in Algorithm~\ref{alg:apscd-global}.
	\label{thm:apscd_convergence}
\end{theorem}

Taking a closer look at Theorem~\ref{thm:apscd_convergence}, we can properly choose the step size $\eta_k$ as a constant value and obtain the following results on convergence rate:
\begin{corollary}
	Let the step length be a constant, i.e., 
	{\small
	\begin{equation}
	    \eta := \sqrt{\frac{(\Psi(x_1) - \Psi(x^*))M N}{LK\sigma^2}}.
	\end{equation}
	}
	If the delay bound $T$ satisfies
	{
	\begin{equation}
	    K \geq \frac{128(\Psi(x_1) - \Psi(x_*)) N L}{M^3 \sigma^2} (T+1)^4,
	    \label{eq:T_bound}
	\end{equation}
	}
	then the output of Algorithm~\ref{alg:apscd-global} satisfies the following ergodic convergence rate as
	{\small
	\begin{equation}
	\begin{split}
		&\quad \mathop{\min}_{k=1,\ldots,K} \mathbb{E}[\norm{P(x_k, g_k, \eta_k)}^2] \\
		&\leq \frac{1}{K}\sum_{k=1}^K \mathbb{E}[\norm{P(x_k, g_k, \eta_k)}^2] \\
		&\leq 32\sqrt{\frac{2(\Psi(x_1) - \Psi(x_*))LM\sigma^2}{KN}}.
	\end{split}
	\label{eq:convergence_1}
	\end{equation}
	}
	\label{corr:apscd_convergence}
\end{corollary}

\textbf{Remark 1. }(Linear speedup w.r.t. the staleness)
	When the maximum delay $T$ is bounded by $O(K^{1/4})$, we can see that the gradient mapping $\mathbb{E}[P(x,g,\eta)]$ decreases regardless of $T$, and thus linear speedup is achievable (if other parameters are constants). In other words, we can see that by \eqref{eq:T_bound} and \eqref{eq:convergence_1}, as long as $T$ is no more than $O(K^{1/4})$, the iteration complexity (from a global perspective) to achieve $\epsilon$-optimality is $O(1/\epsilon^2)$, which is independent from $T$. 

\textbf{Remark 2. }(Linear speedup w.r.t. number of workers)
	We note that the delay bound $T$ is roughly proportional to the number of workers, so the total iterations w.r.t. $T$ can be an indicator of convergence w.r.t. the number of workers. As the iteration complexity is $O(1/\epsilon^2)$ to achieve $\epsilon$-optimality, and it is independent from $T$, we can conclude that the total iterations will be shortened to $1/T$ of a single worker's iterations if $\Theta(T)$ workers work in parallel. This shows that our algorithm nearly achieves linear speedup.

\textbf{Remark 3. }(Consistency with ProxSGD)
	When $T=0$, our proposed AsyB-ProxSGD reduces to the vanilla proximal stochastic gradient descent (ProxSGD) (e.g., \cite{ghadimi2016mini}). Thus, the iteration complexity is $O(1/\epsilon^2)$ according to \eqref{eq:convergence_1}, attaining the same result as that in \cite{ghadimi2016mini} \emph{yet without assuming increased minibatch sizes}.




\section{Experiments}
\label{sec:simu}

\begin{figure*}[t]
  \centering
  \subfigure[Iteration vs. Objective]{
    \includegraphics[height=1.5in]{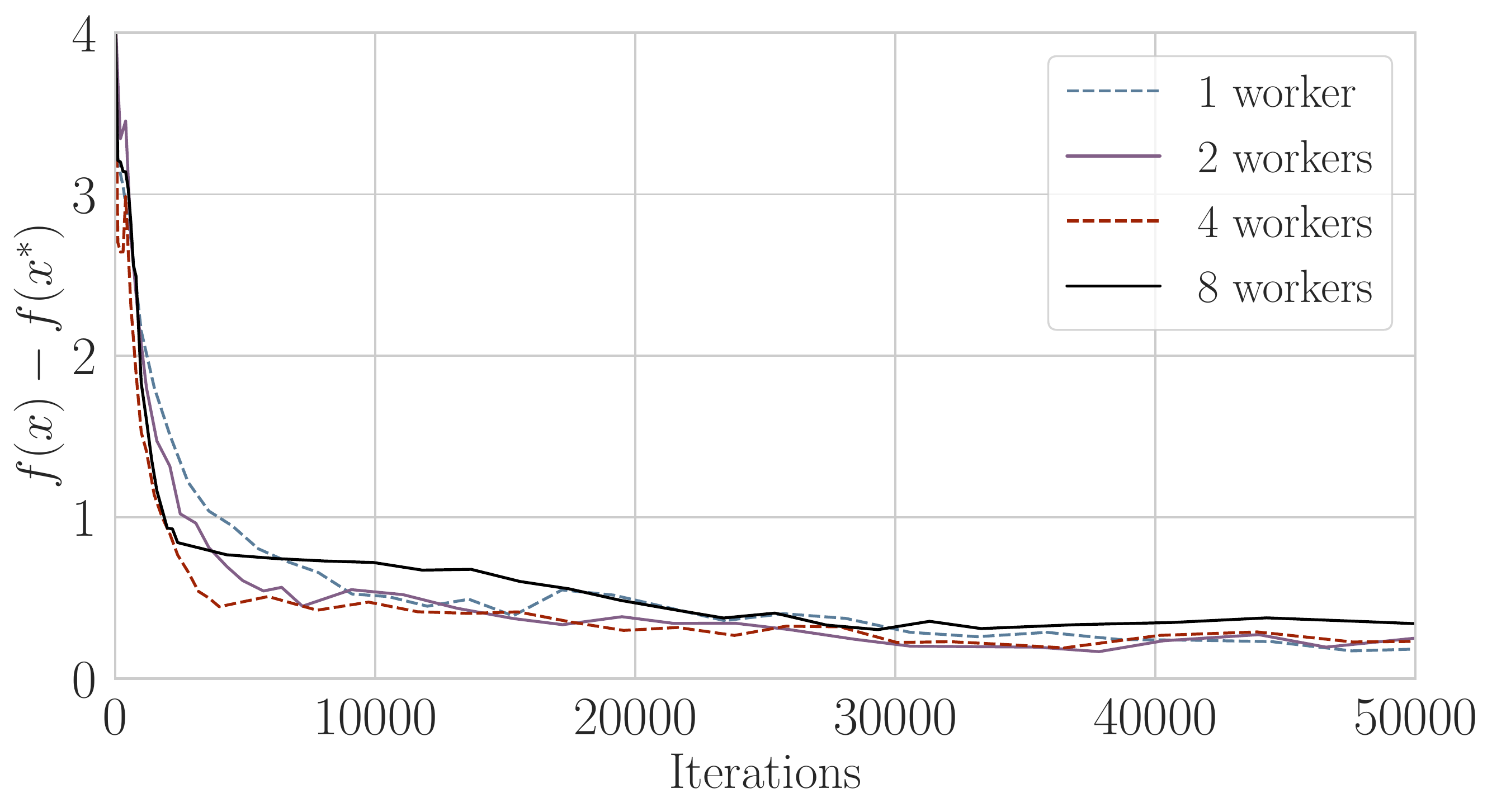}
    \label{fig:obj1}
  }
  \hspace{2mm}
  \subfigure[Time vs. Objective]{
    \includegraphics[height=1.5in]{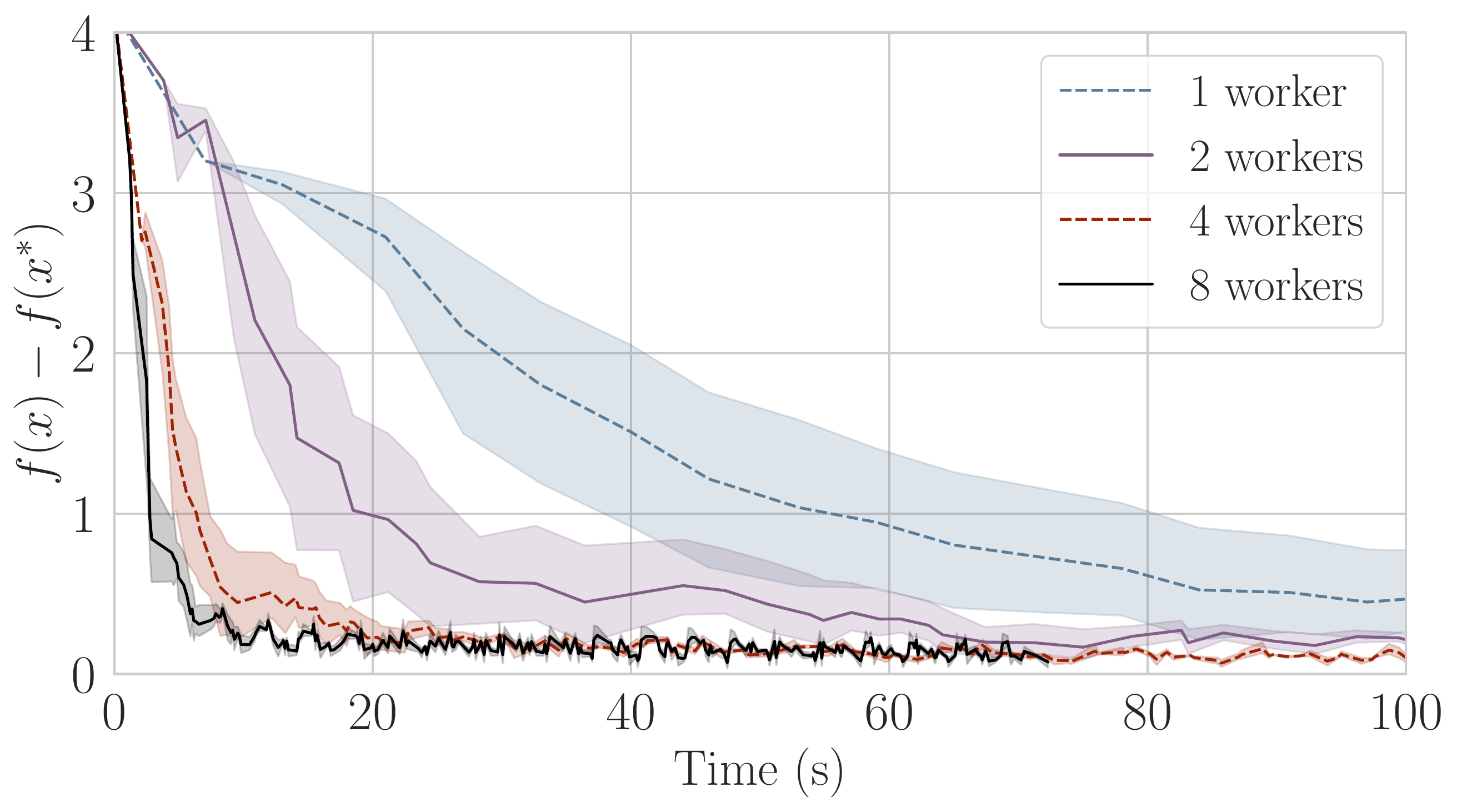}
    \label{fig:obj2}
  }
  \vspace{-3mm}
  \caption{Convergence of AsyB-ProxSGD on the sparse logistic regression problem under different numbers of workers. In this figure, the number of servers is fixed to 8.}
  \vspace{-2mm}
  \label{fig:simu_obj}
\end{figure*}
\vspace{-1mm}

We now present numerical results to confirm that our proposed algorithms can be used to solve the challenging non-convex non-smooth problems in machine learning. 

\textbf{Setup:} 
In our experiments, we consider the sparse logistic regression problem:
{
\begin{equation}
  \mathop{\min}_{x} \frac{1}{n}\sum_{i=1}^n \log(1 + \exp(-b_i \cdot a_i^\top x)) + \lambda_1\norm{x} + \frac{\lambda_2}{2}\norm{x}^2.
  \label{eq:sparse_lr}
\end{equation}
}
The $\ell_1$-regularized logistic regression is widely used for large scale risk minimization. We consider the \texttt{Avazu} dataset \footnote{Available at \url{http://www.csie.ntu.edu.tw/~cjlin/libsvmtools/datasets/}}, which is used in a click-through rate prediction competition jointly hosted by Avazu and Kaggle in 2014. In its training dataset (\texttt{avazu-app}), there are more than 14 million samples, 1 million features, and 4050 million nonzero entries. In other words, both $n$ and $d$ in \eqref{eq:sparse_lr} are large.

We use a cluster of 16 instances on Google Cloud. 
Each server or worker process uses just one core. 
Up to 8 instances serve as server nodes, while the other 8 instances serve as worker nodes. To show the advantage of \emph{asynchronous} parallelism, we set up four experiments adopting 1, 2, 4, and 8 worker nodes, respectively. For all experiments, the whole dataset is shuffled and all workers have a copy of this dataset. When computing a stochastic gradient, each worker takes one minibatch of random samples from its own copy. This way, each sample is used by a worker with an equal probability \emph{empirically} to mimic the scenario of our analysis. 

We consider the \emph{suboptimality gap} as our performance metric, which is defined as the gap between $f(x)$ and $f(x^*)$. Here we estimate the optimal value $\hat{x}$ by performing $5 \times$ as many iterations as needed for convergence. The hyper-parameters are set as follows. For all experiments, the coefficients are set as $\lambda_1=0.1$ and $\lambda_2=0.001$. We set the minibatch size to 8192. The step size is set according to $\eta_k := 0.1 / \sqrt{1.0+k}$ at iteration $k \geq 0$. 

\textbf{Implementation:}

\begin{figure*}[t]
  \centering
  \subfigure[Iteration vs. Objective]{
    \includegraphics[height=1.5in]{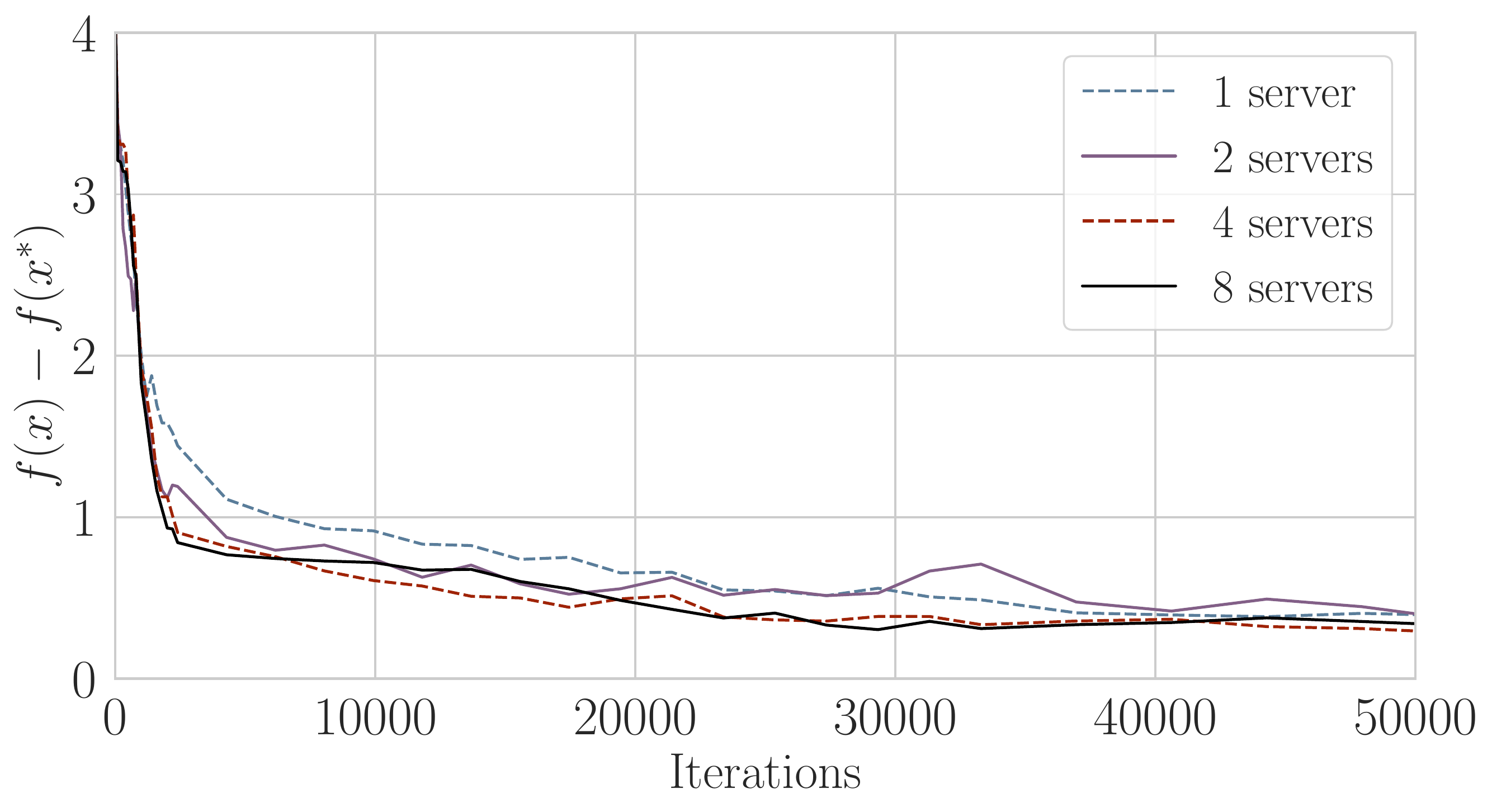}
    \label{fig:obj3}
  }
  \hspace{2mm}
  \subfigure[Time vs. Objective]{
    \includegraphics[height=1.5in]{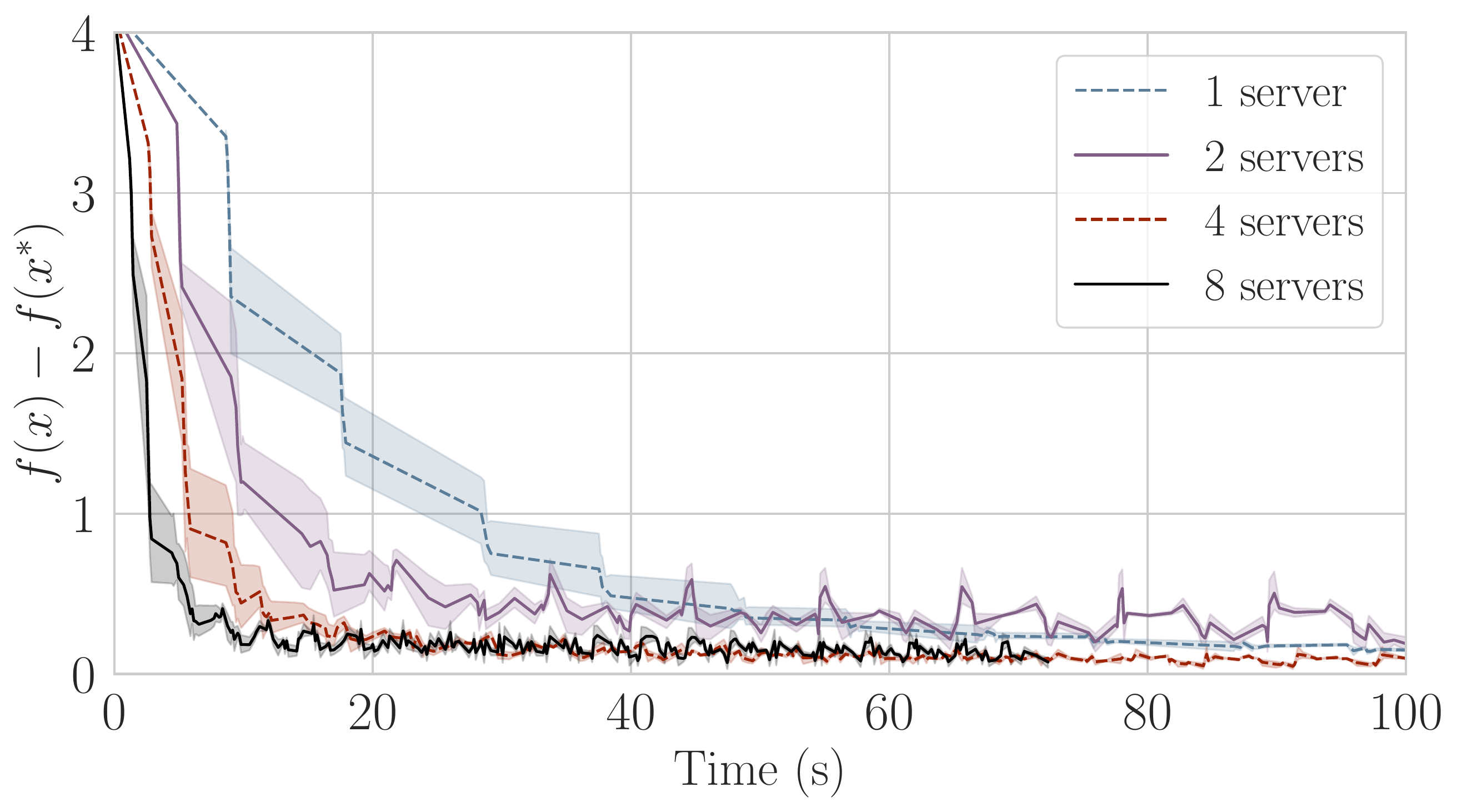}
    \label{fig:obj4}
  }
  \caption{Convergence of AsyB-ProxSGD on the sparse logistic regression problem under different numbers of servers. In this figure, we use 8 workers with different numbers of servers.}
  \label{fig:simu_obj_ps}
\end{figure*}
\vspace{-1mm}

We implemented our algorithm on MXNet \cite{chen2015mxnet}, a flexible and efficient deep learning library with support for distributed machine learning. Due to the sparse nature of the dataset, the model $x$ is stored as a sparse \texttt{ndarray}, and in each iteration, a worker only pulls those blocks of $x$ that are actively related to its sampled minibatch, and then calculates the gradient w.r.t. this minibatch of data, and pushes the gradients for those activated blocks only. 

\textbf{Results:} Empirically, Assumption~\ref{asmp:bound1} (bounded delays) are observed to hold for this cluster. In our experiments, the maximum delay does not exceed 100 iterations unless some worker nodes fail. Fig.~\ref{fig:obj1} and Fig.~\ref{fig:obj2} show the convergence behavior of AsyB-ProxSGD algorithm in terms of objective function values. We can clearly observe the convergence of our proposed algorithm, confirming that asynchrony with tolerable delays can still lead to convergence. In addition, the running time drops in trend when the number of workers increases.

For our proposed AsyB-ProxSGD algorithm, we are particularly interested in two kinds of speedup, namely, iteration speedup and running time speedup. If we need $T_1$ iterations (with $T_1$ sample gradients processed by servers) to achieve a certain suboptimality level using one worker, and $T_p$ iterations to achieve the same level using $p$ workers, then the iteration speedup is defined as $p \times T_1 / T_p$ \cite{lian2015asynchronous}. Note that iterations are counted on the server side, which is actually the number of minibatch gradients are processed by the server. On the other hand, the time speedup is simply defined as the ratio between the running time of using one worker and that of using $p$ workers to achieve the same suboptimality level. We summarize iteration and running time speedup in Table~\ref{table:time_vs_iter}.

\begin{table}[]
\centering
\caption{Iteration speedup and time speedup of AsyB-ProxSGD at the optimality level $10^{-1}$.}
\label{table:time_vs_iter}
\begin{tabular}{c|cccc}
\specialrule{.1em}{.05em}{.05em}
Workers           &  1       &  2       &  4       &  8       \\ \hline
Iteration Speedup & 1.000    & 2.127    & 3.689    & 6.748    \\
Time Speedup      & 1.000    & 1.973    & 4.103    & 8.937    \\
\specialrule{.1em}{.05em}{.05em}
\end{tabular}
\end{table}
\vspace{-3mm}

We further evaluate the relationship between the number of servers and the convergence behavior. Since the model has millions of parameters to be trained, storing the whole model in a single machine can be ineffective. In fact, from Fig.~\ref{fig:simu_obj_ps} we can even see nearly linear speedup w.r.t. the number of servers. The reason here is that, more servers can significantly decrease the length of request queue at the server side. When we have only one server, the blue dashed curve in Fig.~\ref{fig:obj4} looks like a tilt staircase, and further investigation shows that some \texttt{push} requests take too long time to be processed. Therefore, we have to set more than one servers to observe parallel speedup in Fig.~\ref{fig:simu_obj} so that servers are not the bottleneck.

\section{Related Work}
\label{sec:related}

\citet{robbins1951stochastic} propose a classical stochastic approximation algorithm for solving a class of strongly convex problems, which is regarded as the seminal work of stochastic optimization problems. For nonconvex problems, \citet{ghadimi2013stochastic} prove that SGD has an ergodic convergence rate of $O(1/\sqrt{K})$, which is consistent with the convergence rate of SGD for convex problems. To deal with nonsmoothness, proximal gradient algorithm is widely considered and its stochastic variant is heavily studied for convex problems. \citet{duchi2009efficient} show ProxSGD can converge at the rate of $O(1/\mu K)$ for $\mu$-strongly convex objective functions when the step size $\eta_k$ diminishes during iterations. However, for nonconvex problems, rather limited studies on ProxSGD exist so far. To our best knowledge, the seminal work on ProsSGD for nonconvex problems was done by \citet{ghadimi2016mini}, in which the convergence analysis is based on the assumption of an increasing minibatch size.

Updating a single block in each iteration is also referred to as block coordinate methods in the literature. Block coordinate methods for smooth problems with separable, convex constraints \cite{tseng1991rate} and general nonsmooth regularizers \cite{razaviyayn2014parallel,davis2016asynchronous,zhou2016convergence} are proposed. However, the study on \emph{stochastic} coordinate descent is limited and existing work like \cite{liu2015asynchronous} focuses on convex problems. \citet{xu2015block} study block stochastic proximal methods for nonconvex problems. However, they only analyze convergence to stationary points assuming an increasing minibatch size, and the convergence rate is not provided. \citet{davis2016sound} presents a stochastic block coordinate method, which is the closest one with our work in this paper. However, the algorithm studied in \cite{davis2016sound} depends on the use of a noise term with diminishing variance to guarantee convergence. Our convergence results of ProxSGD do not rely on the assumption of increasing batch sizes, variance reduction or the use of additional noise terms.

\section{Concluding Remarks}
\label{sec:conclude}

In this paper, we propose AsyB-ProxSGD as an extension of the proximal stochastic gradient (ProxSGD) algorithm to asynchronous model parallelism setting. Our proposed algorithm aims at solving nonconvex nonsmooth optimization problems involved with large data size and model dimension.  Theoretically, we prove that the AsyB-ProxSGD method has a convergence rate to critical points with the same order as ProxSGD, as long as workers have bounded delay during iterations. Our convergence result does not rely on minibatch size increasing, which is required in all existing works for ProxSGD (and its variants). We further prove that AsyB-ProxSGD can achieve linear speedup when the number of workers is bounded by $O(K^{1/4})$. We implement AsyB-ProxSGD on Parameter Server and experiments on large scale real-world dataset confirms its effectiveness.

\bibliographystyle{abbrvnat}
\bibliography{src/main}

\newpage
\onecolumn
\appendix



\section{Auxiliary Lemmas}
\begin{lemma}[\cite{ghadimi2016mini}]
    For all $y \gets \prox{\eta h}(x - \eta g)$, we have:
    \begin{equation}
        \innerprod{g, y-x} + (h(y) - h(x)) \leq -\frac{\norm{y-x}_2^2}{\eta}.
    \end{equation}
    \label{lem:lem-1}
\end{lemma}
Due to slightly different notations and definitions in \cite{ghadimi2016mini}, we provide a proof here for completeness. We refer readers to \cite{ghadimi2016mini} for more details.
\begin{proof}
    By the definition of proximal function, there exists a $p \in \partial h(y)$ such that:
    \begin{equation*}
        \begin{split}
            \innerprod{g + \frac{y-x}{\eta} + p, x - y} &\geq 0, \\
            \innerprod{g, x-y} &\geq \frac{1}{\eta} \innerprod{y-x, y-x} + \innerprod{p, y-x} \\
            \innerprod{g, x-y} + (h(x) - h(y)) &\geq \frac{1}{\eta} \norm{y-x}_2^2,
        \end{split}
    \end{equation*}
    which proves the lemma.
\end{proof}
By applying the above lemma for each block, we have the following corollary, which is useful in convergence analysis for Algorithm~\ref{alg:abpsgd-global}.
\begin{corollary}
    For all $y_j \gets \prox{\eta h_j}(x_j - \eta g_j)$, we have:
    \begin{equation}
        \innerprod{g_j, y_j-x_j} + (h_j(y_j) - h_j(x_j)) \leq -\frac{\norm{y_j-x_j}_2^2}{\eta}.
    \end{equation}
    \label{cor:lem-1}
\end{corollary}

\begin{lemma}[\cite{ghadimi2016mini}]
    For all $x,g,G \in \real{d}$, if $h: \real{d} \to \real{}$ is a convex function, we have
    \begin{equation}
        \norm{\prox{\eta h}(x - \eta G) - \prox{\eta h}(x - \eta g)} \leq \eta \norm{G-g}.
    \end{equation}
    \label{lem:lem-2}
\end{lemma}
\begin{proof}
    Let $y$ denote $\prox{\eta h}(x - \eta G)$ and $z$ denote $\prox{\eta h}(x - \eta g)$. By definition of the proximal operator, for all $u \in \real{d}$, we have
    \begin{align*}
        \innerprod{G + \frac{y-x}{\eta} + p, u-y} &\geq 0, \\
        \innerprod{g + \frac{z-x}{\eta} + q, u-z} &\geq 0,
    \end{align*}
    where $p \in \partial h(y)$ and $q \in \partial h(z)$. Let $z$ substitute $u$ in the first inequality and $y$ in the second one, we have
    \begin{align*}
        \innerprod{G + \frac{y-x}{\eta} + p, z-y} &\geq 0, \\
        \innerprod{g + \frac{z-x}{\eta} + q, y-z} &\geq 0.
    \end{align*}
    Then, we have
    \begin{align}
        \innerprod{G, z-y} &\geq \innerprod{\frac{y-x}{\eta}, y-z} + \innerprod{p, y-z}, \\
        &= \frac{1}{\eta}\innerprod{y-z, y-z} + \frac{1}{\eta}\innerprod{z-x, y-z} + \innerprod{p, y-z}, \\
        &\geq \frac{\norm{y-z}^2}{\eta} + \frac{1}{\eta}\innerprod{z-x, y-z} +h(y) - h(z),
         \label{eq:lem-2-1}
    \end{align}
    and
    \begin{align}
        \innerprod{g, y-z} &\geq \innerprod{\frac{z-x}{\eta} + q, z-y}, \\
        &= \frac{1}{\eta} \innerprod{z-x, z-y} + \innerprod{q, z-y} \\
        &\geq \frac{1}{\eta} \innerprod{z-x, z-y} + h(z) - h(y). \label{eq:lem-2-2}
    \end{align}
    By adding \eqref{eq:lem-2-1} and \eqref{eq:lem-2-2}, we obtain
    \begin{align*}
        \norm{G-g} \norm{z-y} \geq \innerprod{G-g, z-y} \geq \frac{1}{\eta} \norm{y-z}^2,
    \end{align*}
    which proves the lemma.
\end{proof}
\begin{corollary}
    For all $x_j,g_j,G_j \in \real{d_j}$, we have
    \begin{equation}
        \norm{\prox{\eta h_j}(x_j - \eta G_j) - \prox{\eta h_j}(x_j - \eta g_j)} \leq \eta \norm{G_j - g_j}.
    \end{equation}
    \label{cor:lem-2}
\end{corollary}

\begin{lemma}[\cite{ghadimi2016mini}]
    For any $g_1$ and $g_2$, we have
    \begin{equation}
        \norm{P(x, g_1, \eta) - P(x, g_2, \eta)} \leq \norm{g_1 - g_2}.
    \end{equation}
    \label{lem:lem-3}
\end{lemma}
\begin{proof}
    It can be obtained by directly applying Lemma~\ref{lem:lem-2} and the definition of gradient mapping.
\end{proof}

\begin{corollary}
    Let $P_j(x, g, \eta) := \frac{1}{\eta}(x_j - \prox{\eta h_j}(x_j - \eta g_j))$. Then, for any $G_j$ and $g_j$, we have
    \begin{equation}
        \norm{P_j(x, G, \eta) - P_j(x, g, \eta)} \leq \norm{G- g}.
    \end{equation}
    \label{cor:lem-3}
\end{corollary}

\begin{lemma}[\cite{reddi2016proximal}]
    Suppose we define $y = \prox{\eta h} (x - \eta g)$ for some $g$. Then for $y$, the following inequality holds:
    \begin{equation}
    \begin{split}
        \Psi(y) \leq \Psi(z) + &\innerprod{y-z, \nabla f(x) - g} \\
        &+ \left( \frac{L}{2} - \frac{1}{2\eta} \right) \norm{y-x}^2
         + \left( \frac{L}{2} + \frac{1}{2\eta} \right) \norm{z-x}^2
         - \frac{1}{2\eta} \norm{y-z}^2,
    \end{split}
    \end{equation}
    for all $z$.
    \label{lem:grad_diff}
\end{lemma}

\begin{corollary}
    Suppose we define $y_j = \prox{\eta h_j} (x_j - \eta g_j)$ for some $g_j$, and the index $j$ is chosen among $M$ indices with uniform distribution. For other $j' \neq j$, we assume $y_{j'} = x_{j'}$. Then the following inequality holds:
    \begin{equation}
    \begin{split}
        \Psi(y) &\leq \Psi(z) + \innerprod{\nabla_j f(x) - g_j, y_j - z_j} \\
        &\quad\quad + \left( \frac{L_j}{2} - \frac{1}{2\eta}\right)\norm{y_j-x_j}^2 + \left( \frac{L_j}{2} + \frac{1}{2\eta}\right)\norm{z_j-x_j}^2 - \frac{1}{2\eta}\norm{y_j-x_j}^2.
    \end{split}
    \end{equation}
    for all $z$.
    \label{cor:grad_diff}
\end{corollary}

\begin{proof}
    From the definition of proximal operator, we have
    \begin{align*}
        &\ \quad h_j(y_j) + \innerprod{g_j, y_j-x_j} + \frac{1}{2\eta} \norm{y_j - x_j}^2 + \frac{\eta}{2}\norm{g_j}^2 \\
        &\leq h_j(z_j) + \innerprod{g_j, z_j-x_j} + \frac{1}{2\eta} \norm{z_j - x_j}^2 + \frac{\eta}{2}\norm{g_j}^2 - \frac{1}{2\eta}\norm{y_j - z_j}^2.
    \end{align*}
    By rearranging the above inequality, we have
    \begin{equation}
        h_j(y_j) + \innerprod{g_j, y_j-z_j}  \leq h_j(z_j) + \frac{1}{2\eta}[\norm{z_j - x_j}^2 - \norm{y_j - x_j}^2  - \norm{y_j - z_j}^2].
    \end{equation}
    Since $f$ is $L$-Lipschitz, we have
    \begin{align*}
        f(y) &\leq f(x) + \innerprod{\nabla_j f(x), y_j-x_j} + \frac{L_j}{2}\norm{y_j-x_j}^2, \\
        f(x) &\leq f(z) + \innerprod{\nabla_j f(x), x_j-z_j} + \frac{L_j}{2}\norm{x_j-z_j}^2.
    \end{align*}
    Adding these two inequality we have
    \begin{equation}
        f(y) \leq f(z)  + \innerprod{\nabla_j f(x), y_j-z_j} + \frac{L}{2}[\norm{y_j-x_j}^2 + \norm{z_j-x_j}^2],
    \end{equation}
    and therefore
    \begin{align*}
        \Psi(y) &\leq \Psi(z) + \innerprod{\nabla_j f(x) - g_j, y_j - z_j} \\
        &\quad\quad + \left( \frac{L_j}{2} - \frac{1}{2\eta}\right)\norm{y_j-x_j}^2 + \left( \frac{L_j}{2} + \frac{1}{2\eta}\right)\norm{z_j-x_j}^2 - \frac{1}{2\eta}\norm{y_j-x_j}^2.
    \end{align*}
\end{proof}

\begin{lemma}[Young's Inequality]
    \begin{equation}
        \innerprod{a, b} \leq \frac{1}{2\delta} \norm{a}^2 + \frac{\delta}{2} \norm{b}^2.
    \end{equation}
    \label{lem:young}
\end{lemma}

We recall and define some notations for convergence analysis in the subsequent. We denote $\hat{G}_j^k$ as the average of \emph{delayed} stochastic gradients and $\hat{g}_j^k$ as the average of \emph{delayed} true gradients, respectively:
\begin{align*}
    \hat{G}_j^k &:= \frac{1}{N}\sum_{i=1}^N \nabla_j F(x_{j}^{k-\mathbf{d}_k}; \xi_{k, i}) \\
    \hat{g}_j^k &:= \frac{1}{N}\sum_{i=1}^N \nabla_j f(x_j^{k-\mathbf{d}_k}).
\end{align*}
Moreover, we denote $\delta_j^k := \hat{g}_j^k - \tilde{G}_j^k$ as the difference between these two differences.

\section{Convergence analysis for B-PAPSG}
To simplify notations, we use $j$ instead of $j_k$ in this section. Since we update one block only in each iteration, we define an auxiliary function as follows:
\begin{equation*}
    P_j(x, g, \eta) := \frac{1}{\eta}(x_j - \prox{\eta h_j}(x_j - \eta g_j)),
\end{equation*}
where the variables $x_j$ and $g_j$ take the $j$-th coordinate.

\subsection{Milestone Lemmas}
\begin{lemma}[Descent Lemma]
    \begin{equation}
    \mathbb{E}_j[\Psi(x^{k+1})|\mathcal{F}_k] \leq \mathbb{E}_j[\Psi(x^k)|\mathcal{F}_k] - \frac{\eta_k - 4L_{\max}\eta_k^2}{2M} \norm{P(x^k, \hat{g}^k, \eta_k)}^2 
      + \frac{\eta_k}{2M}\norm{g^k-\hat{g}^k}^2  + \frac{L\eta_k^2}{MN}\sigma^2.
    \end{equation}
    \label{lem:desc_2}
\end{lemma}

\begin{lemma}
    Suppose we have a sequence $\{x^k\}$ by Algorithm~\ref{alg:apscd-global}. Then, we have
    \begin{equation}
        \mathbb{E}[\norm{x^k - x^{k-\tau}}^2] \leq \frac{2T\sum_{l=1}^T \eta_{k-l}^2}{MN} \sigma^2 + 2\Norm{\sum_{l\in K(\bm{\tau}(k))} \eta_{k-l} P_{j_{k-l}}(x^{k-l}, \hat{g}^{k-l}, \eta_{k-l}) }^2
    \end{equation}
    \label{lem:xk_diff_2}
\end{lemma}

\begin{lemma}
    Suppose we have a sequence $\{x^k\}$ by Algorithm~\ref{alg:apscd-global}. Then, we have
    \begin{equation}
        \mathbb{E}[\norm{g_j^k - \hat{g}_j^k}^2] \leq \frac{2L^2 T\sum_{l=1}^T \eta_{k-l}^2}{M^2 N} \sigma^2 + \frac{2L^2 T}{M^2} \sum_{l=1}^T \eta_{k-l}^2\norm{ P(x^{k-l}, \hat{g}^{k-l}, \eta_{k-l}) }^2
    \end{equation}
    \label{lem:gk_diff_2}
\end{lemma}

\subsection{Proof of Theorem~\ref{thm:apscd_convergence}}
\begin{proof}
We have
\begin{equation*}
    \norm{P_j(x_k, \hat{g}_k, \eta_k)}^2 \geq \frac{1}{2} \norm{P_j(x_k, g_k, \eta_k)}^2 - \norm{g_k - \hat{g}_k}^2.
\end{equation*}
When we have $\eta_k \leq \frac{1}{8L_{\max}}$, we can apply the above equation following Lemma~\ref{lem:desc_2}:
\begin{align*}
    &\quad\ \mathbb{E}_j[\Psi(x^{k+1})|\mathcal{F}_k] \\
    &\leq \mathbb{E}_j[\Psi(x^k)|\mathcal{F}_k] - \frac{\eta_k - 4L_{\max}\eta_k^2}{2M} \norm{P(x^k, \hat{g}^k, \eta_k)}^2 
      + \frac{\eta_k}{2M}\norm{g^k-\hat{g}^k}^2  + \frac{L\eta_k^2}{MN}\sigma^2\\
    &\leq \mathbb{E}_j[\Psi(x^k)|\mathcal{F}_k] - \frac{\eta_k - 8L_{\max}\eta_k^2}{4M} \norm{P(x^k, \hat{g}^k, \eta_k)}^2 
       - \frac{\eta_k}{4M}\norm{P(x^k, \hat{g}^k, \eta_k)}^2 \\
       &\quad\ + \frac{\eta_k}{2M}\norm{g^k-\hat{g}^k}^2  + \frac{L\eta_k^2}{MN}\sigma^2 \\
    &\leq \mathbb{E}_j[\Psi(x^k)|\mathcal{F}_k] - \frac{\eta_k - 8L_{\max}\eta_k^2}{8M} \norm{P(x^k, g^k, \eta_k)}^2 + \frac{3\eta_k}{4M} \norm{g^k - \hat{g}^k}^2 \\
    &\quad\ - \frac{\eta_k}{4M}\norm{P(x^k, \hat{g}^k, \eta_k)}^2 + \frac{L\eta_k^2}{MN}\sigma^2 
\end{align*}
By Lemma~\ref{lem:gk_diff_2}, we have
\begin{align*}
    &\quad\ \mathbb{E}_j[\Psi(x^{k+1})|\mathcal{F}_k] \\
    &\leq \mathbb{E}_j[\Psi(x^k)|\mathcal{F}_k] - \frac{\eta_k - 8L_{\max}\eta_k^2}{8M} \norm{P(x^k, g^k, \eta_k)}^2 - \frac{\eta_k}{4M}\norm{P(x^k, \hat{g}^k, \eta_k)}^2 + \frac{L\eta_k^2}{MN}\sigma^2 \\
    &\quad + \frac{3\eta_k}{4M} \left( \frac{2L^2 T\sum_{l=1}^T \eta_{k-l}^2}{M^2 N} \sigma^2 + \frac{2L^2 T}{M^2} \sum_{l=1}^T \eta_{k-l}^2\norm{ P(x^{k-l}, \hat{g}^{k-l}, \eta_{k-l}) }^2 \right) \\
    &\leq \mathbb{E}_j[\Psi(x^k)|\mathcal{F}_k] - \frac{\eta_k - 8L_{\max}\eta_k^2}{8M} \norm{P(x^k, g^k, \eta_k)}^2  + \left( \frac{L\eta_k^2}{MN} + \frac{3\eta_kL^2 T\sum_{l=1}^T \eta_{k-l}^2}{2M^3 N} \right) \sigma^2 \\
    &\quad + \frac{3\eta_kL^2T}{2M^3} \sum_{l=1}^T \eta_{k-l}^2 \norm{P_{j_{k-l}}(x^{k-l}, \hat{g}^{k-l}, \eta_{k-l}) }^2 - \frac{\eta_k}{4M}\norm{P(x^k, \hat{g}^k, \eta_k)}^2
\end{align*}
By taking telescope sum, we have
\begin{align*}
    &\quad\ \sum_{k=1}^K \frac{\eta_k - 8L_{\max}\eta_k^2}{8M} \norm{P(x^k, g^k, \eta_k)}^2 \\
    &\leq \Psi(x^1) - \Psi(x^*) + \sum_{k=1}^K \left( \frac{L\eta_k^2}{MN} + \frac{3\eta_kL^2 T\sum_{l=1}^T \eta_{k-l}^2}{2M^3 N} \right) \sigma^2 \\
    &\quad - \sum_{k=1}^K (\frac{\eta_k}{4M} - \frac{3\eta_k^2 L^2}{2M^3} \sum_{l=1}^{l_k} \eta_{k+l}^2 ) \norm{P(x^k, \hat{g}^k, \eta_k)}^2 \\
    &\leq \Psi(x^1) - \Psi(x^*) + \sum_{k=1}^K \left( \frac{L\eta_k^2}{MN} + \frac{3\eta_kL^2 T\sum_{l=1}^T \eta_{k-l}^2}{2M^3 N} \right) \sigma^2,
\end{align*}
where the last inequality follows from the assumption that $6\eta_{k}^2L^2\sum_{l=1}^T \eta_{k+l} \leq M^2$, and now we prove Theorem~\ref{thm:apscd_convergence}.

\end{proof}

\subsection{Proof of Corollary~\ref{corr:apscd_convergence}}
\begin{proof}
Since the learning rate $\eta_k:=\eta$ is a constant, we apply it to Theorem~\ref{thm:apscd_convergence} and we have:
\begin{equation}
    \frac{1}{K} \sum_{k=1}^K \mathbb{E}[\norm{P(x^k, g^k, \eta_k)}^2] \leq \frac{16M(\Psi(x^1) - \Psi(x^*))}{K\eta} + \frac{16M}{\eta} \left( \frac{L\eta^2}{MN} + \frac{3L^2 T^2 \eta^3}{2M^3 N} \right) \sigma^2.
    \label{eq:corr2-1}
\end{equation}
Following conditions in Corollary~\ref{corr:apscd_convergence}, we have
\begin{equation*}
    \eta \leq \frac{M^2}{16L(T+1)^2},
\end{equation*}
and thus we have
\begin{align*}
    \frac{3L T^2 \eta}{2M^2} &\leq \frac{3 M^2 T^2}{2M^2 \cdot 16 (T+1)^2} \leq 1, \\
    \frac{3L^2 T^2 \eta^3}{2M^2} &\leq L\eta^2.
\end{align*}
Then, we can estimate \eqref{eq:corr2-1} from the above inequality as follows:
\begin{align*}
    \frac{1}{K} \sum_{k=1}^K \mathbb{E}[\norm{P(x^k, g^k, \eta_k)}^2] &\leq \frac{16M(\Psi(x^1) - \Psi(x^*))}{K\eta} + \frac{16M}{\eta} \left( \frac{L\eta^2}{MN} + \frac{3L^2 T^2 \eta^3}{2M^3 N} \right) \sigma^2 \\
    &\leq \frac{16M(\Psi(x^1) - \Psi(x^*))}{K\eta} + \frac{32L\eta}{N} \sigma^2 \\
    &= 32 \sqrt{\frac{2(\Psi(x^1) - \Psi(x^*))LM(T+1)\sigma^2}{KN}},
\end{align*}
which proves the corollary.
\end{proof}

\subsection{Proof of Milestone Lemmas}
\begin{proof}[Proof of Lemma~\ref{lem:desc_2}]
Recall Corollary~\ref{cor:grad_diff}:
\begin{align*}
    \Psi(x^{k+1}) &\leq \Psi(\bar{x}^{k+1}) + \innerprod{\nabla_j f(x^k) - \hat{G}_j^k, x_j^{k+1} - \bar{x}_j^{k+1}} \\
    &\quad\quad + \left( \frac{L_j}{2} - \frac{1}{2\eta}\right)\norm{y_j-x_j}^2 + \left( \frac{L_j}{2} + \frac{1}{2\eta}\right)\norm{z_j-x_j}^2 - \frac{1}{2\eta}\norm{y_j-x_j}^2.
\end{align*}

Now we turn to bound $\Psi(\bar{x}_j^{k+1})$ as follows:
\begin{displaymath}
    \begin{split}
        &\quad\ f(\bar{x}^{k+1}) \\
        &\leq f(x^k) + \innerprod{\nabla_j f(x^k), \bar{x}_j^{k+1} - x_j^k} + \frac{L_j}{2}\norm{\bar{x}_j^{k+1} - x_j^k}^2 \\
        &= f(x^k) + \innerprod{g_j^k, \bar{x}_j^{k+1} - x_j^k} + \frac{\eta_k^2 L_j}{2}\norm{P_j(x^k, \hat{g}^k, \eta_k)}^2 \\
        &= f(x^k) + \innerprod{\hat{g}_j^k, \bar{x}_j^{k+1} - x_j^k}+ \innerprod{g_j^k - \hat{g}_j^k, \bar{x}_j^{k+1} - x_j^k} + \frac{\eta_k^2 L_j}{2}\norm{P_j(x^k, \hat{g}^k, \eta_k)}^2 \\
        &= f(x^k) - \eta_k \innerprod{\hat{g}_j^k, P_j(x^k, \hat{g}_j^k, \eta_k)} + \innerprod{g_j^k - \hat{g}_j^k, \bar{x}_j^{k+1} - x_j^k} + \frac{\eta_k^{2} L_j}{2}\norm{P_j(x^k, \hat{g}^k, \eta_k)}^2 \\
        &\leq  f(x^k) - [\eta_k \norm{P_j(x^k, \hat{g}^k, \eta_k)}^2 + h_j(\bar{x}_j^{k+1}) - h_j(x_j^k)] \\
            &\quad\ + \innerprod{g_j^k - \hat{g}_j^k, \bar{x}_j^{k+1} - x_j^k} + \frac{\eta_k^2 L_j}{2}\norm{P_j(x^k, \hat{g}^k, \eta_k)}^2,
    \end{split}
\end{displaymath}
where the last inequality follows from Corollary~\ref{cor:lem-1}. By rearranging terms on both sides, we have
\begin{equation}
    \Psi(\bar{x}^{k+1}) \leq \Psi(x^k) - (\eta_k - \frac{\eta_k^2 L_j}{2}) \norm{P_j(x^k, \hat{g}^k, \eta_k)}^2 + \innerprod{g_j^k - \hat{g}_j^k, \bar{x}_j^{k+1} - x_j^k}
    \label{eq:baryk_xk_async}
\end{equation}
Taking the summation of \eqref{eq:yk_baryk_async} and \eqref{eq:baryk_xk_async}, we have
\begin{displaymath}
\begin{split}
    &\quad\ \Psi(x^{k+1}) \\
    &\leq \Psi(x^k) + \innerprod{\nabla_j f(x^k) - \hat{G}_j^k, x_j^{k+1} - \bar{x}_j^{k+1}} \\
    &\quad + \left( \frac{L_j}{2} - \frac{1}{2\eta}\right)\norm{y_j-x_j}^2 + \left( \frac{L_j}{2} + \frac{1}{2\eta_k}\right)\norm{z_j-x_j}^2 - \frac{1}{2\eta_k}\norm{y_j-x_j}^2 \\
    &\quad - (\eta_k - \frac{\eta_k^2 L_j}{2}) \norm{P_j(x^k, \hat{g}_k, \eta_k)}^2 + \innerprod{g_j^k - \hat{g}_j^k, \bar{x}_j^{k+1} - x_j^k} \\
    &= \Psi(x^k) + \innerprod{\nabla_j f(x^k) - \hat{g}_j^k, x_j^{k+1} - \bar{x}_j^{k+1}} + \innerprod{\hat{g}_j^k-\hat{G}_j^k, x_j^{k+1} - \bar{x}_j^{k+1}} \\
    &\quad + \left( \frac{L_j\eta_k^2}{2} - \frac{\eta_k}{2} \right) \norm{P_j(x^k, \hat{G}^k, \eta_k)}^2
     + \left( \frac{L_j\eta_k^2}{2} + \frac{\eta_k}{2} \right) \norm{P_j(x^k, \hat{g}^k, \eta_k)}^2 \\
    &\quad - \frac{1}{2\eta_k} \norm{x_j^{k+1}-\bar{x}_j^{k+1}}^2  - (\eta_k - \frac{\eta_k^2 L_j}{2}) \norm{P_j(x^k, \hat{g}^k, \eta_k)}^2 \\
    &= \Psi(x^k) + \innerprod{x_j^{k+1}-x_j^k, g_j^k - \hat{g}_j^k} + \innerprod{x_j^{k+1}-\bar{x}_j^{k+1}, \delta_j^k}
     + \frac{L_j\eta_k^2 - \eta_k}{2} \norm{P_j(x^k, \hat{G}^k, \eta_k)}^2 \\
    &\quad + \frac{2L_j\eta_k^2-\eta_k}{2} \norm{P_j(x^k, \hat{g}^k, \eta_k)}^2 - \frac{1}{2\eta_k} \norm{x_j^{k+1}-\bar{x}_j^{k+1}}^2.
\end{split}
\end{displaymath}
By taking the expectation on condition of filtration $\mathcal{F}_k$ and $j$, we have the following equation according to Assumption~\ref{asmp:unbias_grad}:
\begin{equation}
\begin{split}
&\quad\ \mathbb{E}_j[\Psi(x^{k+1})|\mathcal{F}_k] \\
&\leq \mathbb{E}_j[\Psi(x^k)|\mathcal{F}_k] + \frac{1}{M}\mathbb{E}[\innerprod{x^{k+1}-x^k, g^k-\hat{g}_k}|\mathcal{F}_k] + \frac{L_{\max}\eta_k^2 - \eta_k}{2M} \mathbb{E}[\norm{P(x^k, \hat{G}^k, \eta_k)}^2|\mathcal{F}_k] \\
&\quad + \frac{2L_{\max}\eta_k^2-\eta_k}{2M} \norm{P(x^k, \hat{g}^k, \eta_k)}^2 - \frac{1}{2M\eta_k} \norm{x^{k+1}-\bar{x}^{k+1}}^2.
\end{split}
\end{equation}
Therefore, we have
\begin{equation*}
\begin{split}
    &\quad\ \mathbb{E}_j[\Psi(x^{k+1})|\mathcal{F}_k] \\
    &\leq \mathbb{E}_j[\Psi(x^k)|\mathcal{F}_k] + \frac{1}{M}\mathbb{E}[\innerprod{x^{k+1}-x^k, g^k-\hat{g}^k}|\mathcal{F}_k] + \frac{L_{\max}\eta_k^2 - \eta_k}{2M} \mathbb{E}[\norm{P(x^k, \hat{G}^k, \eta_k)}^2|\mathcal{F}_k] \\
      &\quad + \frac{2L_{\max}\eta_k^2-\eta_k}{2M} \norm{P(x^k, \hat{g}^k, \eta_k)}^2 - \frac{1}{2M\eta_k} \norm{x^{k+1}-\bar{x}^{k+1}}^2 \\
    &\leq \mathbb{E}_j[\Psi(x^k)|\mathcal{F}_k] + \frac{\eta_k}{2M}\norm{g^k-\hat{g}^k}^2 + \frac{L_{\max}\eta_k^2}{2M} \mathbb{E}[\norm{P(x^k, \hat{G}^k, \eta_k)}^2|\mathcal{F}_k]
      + \frac{2L_{\max}\eta_k^2-\eta_k}{2M} \norm{P(x^k, \hat{g}^k, \eta_k)}^2  \\
    &\leq \mathbb{E}_j[\Psi(x^k)|\mathcal{F}_k] \\
    &\quad\ - \frac{\eta_k - 4L_{\max}\eta_k^2}{2M} \norm{P(x^k, \hat{g}^k, \eta_k)}^2 
      + \frac{\eta_k}{2M}\norm{g^k-\hat{g}^k}^2  + \frac{L\eta_k^2}{MN}\sigma^2.
\end{split}
\end{equation*} 
\end{proof}

\begin{proof}[Proof of Lemma~\ref{lem:xk_diff_2}]
    \begin{align*}
    \norm{x^k - x^{k-\tau}}^2 &= \Norm{\sum_{l\in K(\bm{\tau}(k))} x^{k-l+1} - x^{k-l}}^2 \\
    &= \Norm{\sum_{l\in K(\bm{\tau}(k))} \eta_{k-l} P_{j_{k-l}}(x^{k-l}, \hat{G}^{k-l}, \eta_{k-l}) }^2 \\
    &\leq 2\Norm{\sum_{l\in K(\bm{\tau}(k))} \eta_{k-l} (P_{j_{k-l}}(x^{k-l}, \hat{G}^{k-l}, \eta_{k-l})-P_{j_{k-l}}(x^{k-l}, \hat{g}^{k-l}, \eta_{k-l})) }^2 \\
    &\quad+ 2\Norm{\sum_{l\in K(\bm{\tau}(k))} \eta_{k-l} P_{j_{k-l}}(x^{k-l}, \hat{g}^{k-l}, \eta_{k-l}) }^2  \\
    &\leq 2T\sum_{l\in K(\bm{\tau}(k))} \eta_{k-l}^2 \Norm{P_{j_{k-l}}(x^{k-l}, \hat{G}^{k-l}, \eta_{k-l})-P_{j_{k-l}}(x^{k-l}, \hat{g}^{k-l}, \eta_{k-l})}^2 \\
    &\quad + 2\Norm{\sum_{l\in K(\bm{\tau}(k))} \eta_{k-l} P_{j_{k-l}}(x^{k-l}, \hat{g}^{k-l}, \eta_{k-l}) }^2  \\
    &\leq 2T\sum_{l=1}^T \eta_{k-l}^2 \norm{\hat{G}^{k-l} - \hat{g}^{k-l}}^2 + 2\Norm{\sum_{l\in K(\bm{\tau}(k))} \eta_{k-l} P_{j_{k-l}}(x^{k-l}, \hat{g}^{k-l}, \eta_{k-l}) }^2  \\
    &\leq \frac{2T\sum_{l=1}^T \eta_{k-l}^2}{MN} \sigma^2 + 2\Norm{\sum_{l\in K(\bm{\tau}(k))} \eta_{k-l} P_{j_{k-l}}(x^{k-l}, \hat{g}^{k-l}, \eta_{k-l}) }^2
\end{align*}
\end{proof}

\begin{proof}[Proof of Lemma~\ref{lem:gk_diff_2}]
\begin{align*}
    \mathbb{E}[\norm{g_j^k - \hat{g}_j^k}^2] &= \Norm{\frac{1}{N}\sum_{i=1}^N g_j^k - \hat{g}_j^{k-\tau(k,i)} }^2 \\
    &\leq \frac{1}{N} \sum_{i=1}^N \norm{g_j^k - \hat{g}_j^{k-\tau(k,i)} }^2 \\
    &\leq \frac{1}{MN} \sum_{i=1}^N \norm{g^k - \hat{g}^{k-\tau(k,i)} }^2  \\
    &\leq \frac{L^2}{MN} \sum_{i=1}^N \norm{x^k - \hat{x}^{k-\tau(k,i)} }^2 \\
    &\leq \frac{L^2}{MN} \sum_{i=1}^N \left( \frac{2T\sum_{l=1}^T \eta_{k-l}^2}{MN} \sigma^2 + 2 \Norm{\sum_{l=1}^T \eta_{k-l} P_{j_{k-l}}(x^{k-l}, \hat{g}^{k-l}, \eta_{k-l}) }^2 \right) \\
    &\leq \frac{2L^2 T\sum_{l=1}^T \eta_{k-l}^2}{M^2 N} \sigma^2 + \frac{2L^2 T}{M^2} \sum_{l=1}^T \eta_{k-l}^2\norm{ P(x^{k-l}, \hat{g}^{k-l}, \eta_{k-l}) }^2
\end{align*}
\end{proof}

\end{document}